\documentclass[9pt,conference]{IEEEtran}

\ifCLASSINFOpdf

\else

\fi

\usepackage{amssymb}
\usepackage{amsthm}
\usepackage{amsmath}
\usepackage{graphicx}

\newtheorem{lemma}{Lemma}
\newtheorem{prop}{Proposition}

\begin{document}

\title{Theoretically informed selection of latent activation in autoencoder based recommender systems}

\author{\IEEEauthorblockN{Aviad Susman\IEEEauthorrefmark{1}}
\IEEEauthorblockA{\IEEEauthorrefmark{1}Department of Genetics and Genomic Sciences, Icahn School of Medicine at Mount Sinai, New York, NY 10029\\ Email: aviad.susman@mssm.edu}}

\maketitle
\IEEEpeerreviewmaketitle

\section*{Abstract}
Autoencoders may lend themselves to the design of more accurate and computationally efficient recommender systems by distilling sparse high-dimensional data into dense lower-dimensional latent representations. However, designing these systems remains challenging due to the lack of theoretical guidance. This work addresses this by identifying three key mathematical properties that the encoder in an autoencoder should exhibit to improve recommendation accuracy: (1) dimensionality reduction, (2) preservation of similarity ordering in dot product comparisons, and (3) preservation of non-zero vectors. Through theoretical analysis, we demonstrate that common activation functions, such as ReLU and tanh, cannot fulfill these properties jointly within a generalizable framework. In contrast, sigmoid-like activations emerge as suitable choices for latent activations. This theoretically informed approach offers a more systematic method for hyperparameter selection, enhancing the efficiency of model design.

\section{Introduction}
A general problem in deep learning is the effective design of neural networks. Once an overall method has been chosen for the problem at hand (convolutional model, autoencoder, recurrent model, etc.), further decisions regarding network depth, layer size, and activation functions - among others - remain [Goodfellow]. While an optimal architecture can be estimated with various hyperparameter optimization techniques [Bergstra, Snoek], theoretical guarantees regarding which design choices may be optimal or not, are rare [Zoph]. However, if appropriate mathematical constraints on the architecture are extracted from the desired behavior of the ultimate model, these conditions can be leveraged to make informed and rigorous design choices. Examples of this include recommender systems that make use of autoencoders on user and item feature vectors. Rather than recommending items to users based on features extracted via matrix factorization and other domain relevant processing, these systems pass the user and item feature vectors to an autoencoder for further compression and information distillation before computing comparisons via a dot product or cosine similarity measure on the resultant latent representations (Figure 1). Such systems are used in various application areas ranging from drug discovery [Way] to content recommendation [Cao]. The encoder networks within the autoencoders that are used in these systems have three desired mathematical properties that are inspired by the ultimate task of dot product similarity based comparison. These properties can be leveraged to mathematically determine further properties that may contribute to theoretically informed decision making within their hyperparameter design. Specifically, in this work, we show that three desirable mathematical properties of encoders within recommender systems cannot jointly be fulfilled by a generalizable model that uses a $ReLU$-like or $tanh$-like activation in its latent layer. Among the common choices of activation functions, those that are $sigmoid$-like remain viable.

\begin{figure}[ht]
\centering
\includegraphics[width=0.5\textwidth]{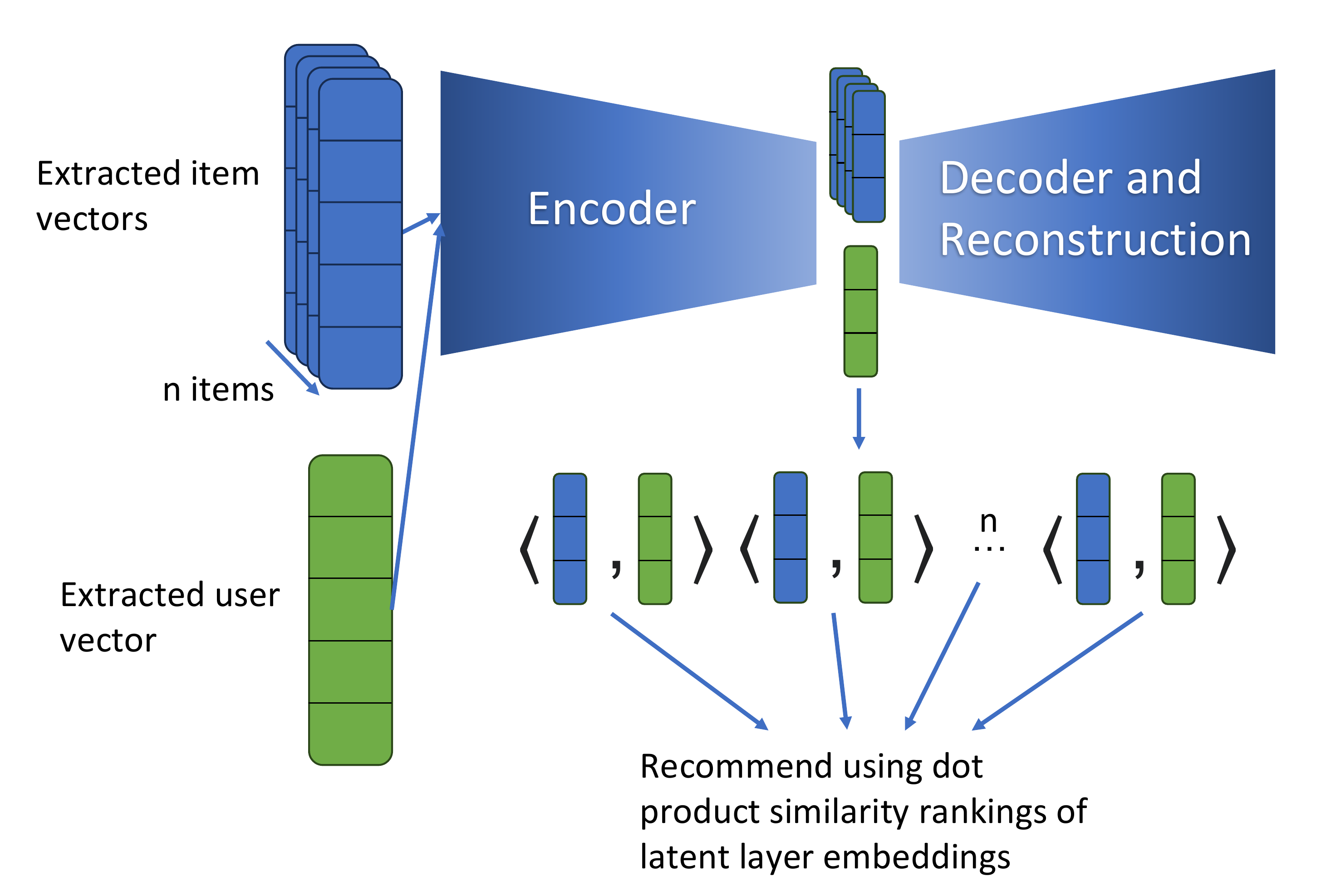}
\caption{An autoencoder used within a recommender system for decreased computational complexity and increased accuracy.}
\label{fig:autoencoder}
\end{figure}

\section{Desired mathematical behavior}
If the encoder network of an autoencoder within a recommender system is a function $enc: \mathbb{R}^n\to\mathbb{R}^m$, then we desire that:

\vspace{0.1 in}
\begin{enumerate}
  \item $m \leq n$.

This is a defining property of autoencoders and not unique to the system we are considering. Dot product similarity based recommender systems may elect to utilize autoencoders that reduce the dimensionality of their inputs in order to avoid comparing items on extraneous features that the encoder can zero out. This may contribute to increased accuracy. Additionally, they are used to reduce computational complexity when the number of items being compared is large or their representations are long vectors [Duong].
\newpage 
  \item If $x,u,v\in\mathbb{R}^n$ and $\langle x, u\rangle\leq\langle x,v\rangle$ then $$\langle enc(x), enc(u)\rangle\leq\langle enc(x),enc(v)\rangle.$$

Here, $x$ can be thought of as a user vector within a recommender system and $u$ and $v$ as item vectors that we may want to recommend to used $x$ based on their dot product similarities. The property of preserving the ordering of similarities to a fixed target vector is important for ensuring that the encoder produces representations that are still useful to the ultimate recommendation task [Li, Liang, Noshad]. 
    
    \item If $x\neq0$ then $enc(x)\neq0$.

This property is essential to the learning task as it maintains the function of dot product similarity based recommendation. Preserving non-zero vectors prevents the similarity rankings for that input from collapsing [Zhang].
\end{enumerate}

\section{Results}
\begin{lemma}
Let $f: \mathbb{R}^n \to \mathbb{R}^m$ be a function that preserves orthogonality. That is, $\langle u,v\rangle = 0 \implies \langle f(u),f(v)\rangle = 0$ for all $u,v\in\mathbb{R}^n$. Additionally, $f$ preserves non-zero vectors namely $x\neq 0\implies f(x) \neq 0$ for all $x\in\mathbb{R}^n$. Then $m\geq n$.
\end{lemma}

\noindent \begin{proof}
    Let $\{b_i\}_{i=1}^n$ be any orthogonal basis of $\mathbb{R}^n$. If $f$ preserves orthogonality then the set $\{f(b_i)\}_{i=1}^n$ is a set of mutually orthogonal non-zero vectors in $\mathbb{R}^m$ and so it is linearly independent via the following argument. If $\{a_i\}_{i=1}^n$ is a set of real numbers and $\sum_{i=1}^n a_if(b_i) = 0$, then for any $1\leq j\leq n$,
    $$\langle f(b_j),\sum_{i=1}^n a_if(b_i)\rangle = \langle f(b_j),0\rangle$$
    and since $f(b_j)$ is non-zero and orthogonal to all other vectors in the sum,
    $$a_j||f(b_j)|| = 0 \implies a_j = 0.$$
    Trivially, to have a set of $n$ linearly independent vectors in $\mathbb{R}^m$, it must be the case that $m\geq n$.
\end{proof}

\vspace{0.1 in}

We employ this lemma to prove a result regarding functions that behave like the autoencoder described above. In the following proposition, our main result, $f$ represents the dimension reducing function implemented by the encoder network, and the vector $x$ is the fixed user vector.

\begin{prop}
    If $f:\mathbb{R}^n\to\mathbb{R}^m$ is a function such that 
    $$\langle x,u\rangle \leq \langle x,v\rangle \implies \langle f(x),f(u)\rangle \leq \langle f(x),f(v)\rangle$$ and
    $$x \neq 0 \implies f(x) \neq 0$$
    for all $x,u,v\in\mathbb{R}^n$, and $m < n$, it must be the case that $f(0) \neq 0$.
\end{prop}
\begin{proof}
Trivially, 
$$\langle 0,0 \rangle \leq \langle 0,u \rangle \leq \langle 0,0 \rangle,$$ and so
\begin{equation}
    \langle f(0), f(u)\rangle = ||f(0)||^2 \text{ for every } u\in\mathbb{R}^n.
\end{equation}
In other words, $Im(f)$ lies in the hyperplane determined by $(f(0), ||f(0)||^2)$. If $u,v$ are two orthogonal vectors in $\mathbb{R}^n$, then
$$\langle u,0 \rangle \leq \langle u,v \rangle \leq \langle u,0 \rangle,$$
and so
$$\langle f(u), f(v)\rangle = \langle f(u), f(0)\rangle \overset{\text{Eq. 1}}{=} ||f(0)||^2.$$ So, if $f(0) = 0$, then, by Lemma 1, $f$ preserves orthogonality and so $m\geq n$. So, in order for $f$ to not preserve orthogonality and be dimension-reducing, it must be nonzero everywhere.
\end{proof}
\vspace{0.1 in}
The encoder network within a trained autoencoder that reduces the dimensionality of its inputs, preserves the order of their dot products relative to a fixed input vector, and preserves non-zero vectors, is an example of a function that satisfies the conditions of Proposition 1. In order to enforce that the encoder is non-zero everywhere (i.e., at zero), a $sigmoid$-like activation in the latent space is an appropriate choice for an activation function, as the $sigmoid$ function is positive. In contrast, other common activation functions, such as $ReLU$ and $tanh$-based activations have zeros, and thus are not design choices that lend themselves to generalizable models.

\section{Conclusion}
We showed that utilizing mathematical constraints of an autoencoder's behavior can help inform the design of its architecture. Specifically, this can be done in regards to the choice of the latent space activation function, which is often made either arbitrarily or sub-optimally with the aid of costly numerical methods. We showed that if the conditions of dimension reduction, inner product order preservation, and preserving non-zero vectors are imposed, then non-vanishing latent space activation functions are appropriate. Further work towards uncovering desirable mathematical behavior of models in more general problem settings may yield additional deep learning design principles. Future work may also include exploring transformations of activations with zeros to restore their suitability for the problem setup being considered. Additionally, we intend to gather empirical evidence to support the theoretical claims of this work. 

\section*{Acknowledgments}
We thank Joseph Colonel, Jamie Bennett, and Gaurav Pandey for their technical assistance. This work was funded by National Institute of Health grant number R01HG011407.

\end{document}